\documentclass[11pt]{article}
\usepackage[margin=1in]{geometry}
\usepackage[T1]{fontenc}
\usepackage{times}
\usepackage{microtype}
\usepackage{amsmath,amssymb,amsthm}
\usepackage[round,authoryear]{natbib}
\usepackage[algo2e,ruled,vlined,linesnumbered]{algorithm2e}
\usepackage{placeins}

\newtheorem{theorem}{Theorem}[section]

\newtheorem{lemma}[theorem]{Lemma}

\theoremstyle{remark}

\theoremstyle{plain}

\newenvironment{proofof}[1]{\begin{proof}[Proof of #1]}{\end{proof}}

\usepackage{amsfonts}
\usepackage{bm}
\usepackage{xcolor}
\usepackage{mathtools}
\usepackage{xspace}
\usepackage{booktabs}
\usepackage{microtype}
\usepackage{hyperref}
\usepackage{cleveref}

\hypersetup{
  colorlinks=true,
  breaklinks=true,
  linkcolor=blue!55!black,
  citecolor=blue!55!black,
  urlcolor=blue!55!black
}
\allowdisplaybreaks

\newcommand{\wh}{\widehat}

\def\E{\mathbb{E}}

\newcommand{\ind}{\mathbb{I}}

\newcommand{\Et}{\mathbb{E}_t}

\newcommand{\one}{\boldsymbol 1}

 \newcommand{\ellhat}{\wh\ell}

\DeclareMathOperator*{\argmax}{argmax}

\newcommand{\order}{\mathcal O}
\newcommand{\otil}{\widetilde{\mathcal O}}

\newcommand{\ip}[2]{\left\langle #1,#2\right\rangle}

\newcommand{\Reg}{\mathrm{Reg}}

\newcommand{\halftick}{\ensuremath{\checkmark\kern-1.1ex\raisebox{.7ex}{\rotatebox[origin=c]{125}{--}}}}

\usepackage{prettyref}
\newcommand{\savehyperref}[2]{\texorpdfstring{\hyperref[#1]{#2}}{#2}}
\newcommand{\pref}[1]{\prettyref{#1}}

\newrefformat{sec}{\savehyperref{#1}{Section~\ref*{#1}}}
\newrefformat{app}{\savehyperref{#1}{Appendix~\ref*{#1}}}
\newrefformat{eq}{\savehyperref{#1}{Eq.~\textup{(\ref*{#1})}}}
\newrefformat{thm}{\savehyperref{#1}{Theorem~\ref*{#1}}}
\newrefformat{lem}{\savehyperref{#1}{Lemma~\ref*{#1}}}
\newrefformat{prop}{\savehyperref{#1}{Proposition~\ref*{#1}}}
\newrefformat{propt}{\savehyperref{#1}{Proposition~\ref*{#1}}}
\newrefformat{assum}{\savehyperref{#1}{Assumption~\ref*{#1}}}
\newrefformat{def}{\savehyperref{#1}{Definition~\ref*{#1}}}
\newrefformat{cor}{\savehyperref{#1}{Corollary~\ref*{#1}}}
\newrefformat{rem}{\savehyperref{#1}{Remark~\ref*{#1}}}
\newrefformat{alg}{\savehyperref{#1}{Algorithm~\ref*{#1}}}
\newrefformat{fig}{\savehyperref{#1}{Figure~\ref*{#1}}}
\newrefformat{tab}{\savehyperref{#1}{Table~\ref*{#1}}}

\definecolor{wine_red}{RGB}{228,48,64}
\definecolor{DSgray}{cmyk}{0,1,0,0}

\RestyleAlgo{ruled}
\SetAlgoVlined
\newrefformat{line}{\savehyperref{#1}{Line~\ref*{#1}}}

\title{Achieving Minimax-Optimal First-Order and Best-of-Both-Worlds Regret Simultaneously in Multi-Armed Bandits}
\author{Mengxiao Zhang \\
University of Iowa \\
\texttt{mengxiao-zhang@uiowa.edu}}
\date{}
\hypersetup{
  pdftitle={Achieving Minimax-Optimal First-Order and Best-of-Both-Worlds Regret Simultaneously in Multi-Armed Bandits},
  pdfauthor={Mengxiao Zhang}
}

\newcommand{\simplex}{\Delta_K}
\renewcommand{\Reg}{\mathsf{Reg}}

\begin{document}
\maketitle

\begin{abstract}
We study first-order and best-of-both-worlds regret guarantees in multi-armed bandits. \citet{pogodin2020first} proposed an algorithm that combines minimax-optimal worst-case regret with a first-order guarantee whose loss-dependent term is $\order(\sqrt{K L^\star}\log T)$, where $T$ denotes the horizon, $L^\star$ denotes the smallest expected cumulative loss among fixed arms, and $K$ denotes the number of arms. However, a logarithmic gap remains between this loss-dependent term and the optimal first-order scale. In this paper, we remove this logarithmic factor by proposing an efficient and parameter-free algorithm that achieves the minimax optimal first-order regret guarantee $\order(\sqrt{K L^\star})$.
Technically, our algorithm is built upon FTRL with a Tsallis entropy regularizer~\citep{zimmert2021tsallis}, with a learning rate adapted to the learner's cumulative loss and a sampling distribution determined by a scalar fixed point. This construction controls the combined estimation and stability cost in terms of the learner's loss. Moreover, we show that the same algorithm achieves a best-of-both-worlds guarantee, with $\order(\sum_{i\ne i^\star}\frac{\log T}{\Delta_i})$ regret for i.i.d. losses with a unique best arm $i^\star$, where $\Delta_i$ denotes the suboptimality gap of arm $i$.
\end{abstract}

\section{Introduction}

We study first-order and best-of-both-worlds regret guarantees in multi-armed bandits (MAB).
In adversarial MAB, a learner has $K$ arms available and interacts with the environment over a horizon $T>2$ where the loss of each arm is arbitrarily decided. At each round, the learner chooses an arm and observes only the loss of the chosen arm. The objective is to minimize the learner's expected cumulative loss relative to the best fixed arm. While the minimax regret is $\Theta(\sqrt{KT})$~\citep{auer2002nonstochastic,audibert2009minimax}, this worst-case guarantee does not adapt to easier loss sequences.

Indeed, one can achieve a much better regret guarantee by exploiting structure in the loss sequence. For example, when the losses are drawn i.i.d. from a fixed distribution, algorithms can achieve logarithmic gap-dependent regret \citep{auer2002finite}. First-order bounds, also called small-loss bounds, provide another form of adaptation by replacing the leading dependence on the horizon with a dependence on the smallest expected cumulative loss $L^\star$ among fixed arms. When $L^\star\ll T$, such a guarantee can be substantially smaller than the minimax bound. In fact, multiple approaches achieve near-optimal first-order guarantees~\citep{allenberg2006hannan,neu2015first,wei2018adaptive,pogodin2020first,ito2021parameter}. Specifically, \citet{pogodin2020first} combine minimax-optimal worst-case regret with a first-order guarantee whose loss-dependent term is $\order(\sqrt{K L^\star}\log T)$, together with additional terms depending on $K$ and $\log T$.\footnote{\citet{pogodin2020first} also show that a better $\order(\sqrt{KL^\star\log T})$ guarantee can be obtained simultaneously with $\order(\sqrt{KT})$ when $T$ is known.} However, none of these guarantees removes the logarithmic factor from the leading first-order term, so these first-order bounds alone do not recover the minimax-optimal $\order(\sqrt{KT})$ regret when $L^\star$ is of order $T$, while the lower-bound construction of \citet{gerchinovitz2016lower} only shows an $\Omega(\sqrt{K L^\star})$ first-order regret lower bound.
Via a Bayesian analysis of Thompson sampling and the minimax theorem, \citet{bubeck2020first} show that some algorithm attains $\order(\sqrt{K\bar{L}}+K\log^2\bar{L})$ regret on every oblivious loss sequence with $L^\star\le\bar{L}$. However, this argument is non-constructive and requires $\bar{L}$ and $T$ in advance. This leaves open whether an explicit algorithm can achieve the minimax-optimal first-order guarantee without such prior knowledge. We therefore ask:
\begin{center}
    \emph{Can one design an efficient algorithm that achieves the minimax-optimal first-order guarantee in adversarial MAB? Moreover, can the same algorithm achieve a best-of-both-worlds guarantee?}
\end{center}

In this paper, we answer this question affirmatively by designing a single algorithm with leading first-order regret $\order(\sqrt{K L^\star})$ in the adversarial environment and $\order(\sum_{i\ne i^\star}\frac{\log T}{\Delta_i})$ pseudo-regret for i.i.d. losses with a unique best arm $i^\star$, where $\Delta_i$ denotes the mean-loss gap of arm $i$ from $i^\star$. Technically, we couple the sampling distribution of Tsallis FTRL~\citep{zimmert2021tsallis} with its implicit exploration estimator~\citep{kocak2014efficient,neu2015explore} through a scalar fixed point. This controls the combined cost of estimation bias and stability in terms of the learner's loss, and a regularization scale based on observed cumulative loss gives the optimal leading first-order term in \pref{sec:lf}. In \pref{sec:bobw}, a refined stability analysis bounds the optimal arm's contribution through suboptimal-arm probabilities, while the estimation bias cancels the gap-dependent part of the stability bound. The resulting regret bound depends on suboptimal-arm probabilities, allowing the same observed-loss scale and the self-bounding identity to yield logarithmic gap-dependent regret. To our knowledge, this is the first efficient algorithm that achieves the minimax optimal first-order regret guarantee for MAB without knowing $L^\star$ or $T$, and it also achieves a best-of-both-worlds guarantee.

\subsection{Related work}

First-order guarantees refine the worst-case bandit bounds of \citet{auer2002nonstochastic,audibert2009minimax} by replacing the horizon dependence with $L^\star$ \citep{allenberg2006hannan,neu2015first}. On the upper-bound side, GREEN \citep{allenberg2006hannan}, FPL-TRIX \citep{neu2015first}, log-barrier mirror descent with a learning rate tuned using $L^\star$ \citep{foster2016games}, and BROAD-OMD \citep{wei2018adaptive} achieve first-order guarantees with loss-dependent terms of $\otil(\sqrt{K L^\star})$.\footnote{The notation $\otil(\cdot)$ suppresses logarithmic factors in $K$ and $T$.} In particular, FPL-TRIX attains $\order(\sqrt{K L^\star\log K})$ without prior knowledge, but its worst-case bound $\order(\sqrt{KT\log K})$ is not minimax optimal, while the log-barrier methods above incur $\order(\sqrt{K L^\star\log T})$. The INF/log-barrier FTRL algorithm of \citet{pogodin2020first} combines minimax-optimal $\order(\sqrt{KT})$ worst-case regret with a first-order guarantee whose loss-dependent term is $\order(\sqrt{K L^\star}\log T)$, but a logarithmic factor remains in this term. First-order guarantees have also been studied in richer models with partial feedback, including bandits with feedback graphs \citep{lykouris2018small,lee2020graph}, contextual bandits \citep{allen2018make,foster2021efficient,olkhovskaya2023first,vanerven2025reduction}, linear bandits \citep{lee2020bias,ito2020linear}, and episodic Markov decision processes (MDPs) \citep{lee2020bias,li2026data}.

Best-of-both-worlds guarantees require a single algorithm to achieve $\otil(\sqrt{KT})$ adversarial pseudo-regret and (poly)logarithmic gap-dependent stochastic pseudo-regret without knowing the environment.
SAO initiated this direction \citep{bubeck2012best}, followed by EXP3++ \citep{seldin2014practical,seldin2017exp3pp} and SAPO \citep{auer2016nearly}.
Tsallis-INF achieves $\order(\sqrt{KT})$ adversarial pseudo-regret and $\order(\sum_{i\ne i^\star}\frac{\log T}{\Delta_i})$ stochastic pseudo-regret with an anytime learning rate \citep{zimmert2021tsallis}.
\citet{ito2021parameter} showed that Tsallis-INF retains logarithmic stochastic regret without a unique optimal arm, and \citet{jin2023improved} extended this to general regularizers.
Follow-the-perturbed-leader (FTPL) also achieves best-of-both-worlds guarantees for multi-armed bandits \citep{honda2023ftpl,lee2024frechet,chen2025geometric,lee2025unbounded,lee2026surrogate}.
This objective has been extended beyond standard MAB to combinatorial semi-bandits \citep{zimmert2019semibandit,chen2026semibandit}, bandits with delayed feedback \citep{masoudian2022delayed,schlisselberg2025delayed}, linear bandits, contextual bandits, and bandits with feedback graphs \citep{dann2023blackbox,ito2022graphs,ito2024spm,kato2025lctsallis}, as well as episodic MDPs \citep{jin2020mdp,jin2021mdp,ito2025aggregate,li2026data}.
Simultaneous first-order and best-of-both-worlds guarantees have been obtained using log-barrier OMD and adaptive FTRL \citep{wei2018adaptive,ito2021parameter,ito2022variance,nguyen2025realtime}.
Our result removes the logarithmic factor from the leading adversarial first-order term while achieving logarithmic gap-dependent stochastic regret.

\section{Preliminaries}

We study multi-armed bandits with $K\ge2$ arms over a horizon $T>2$.
Denote by $[K]\triangleq\{1,\ldots,K\}$ the set of arms and by $\mathcal F_{t-1}$ the $\sigma$-field generated by the learner's past actions, observed losses, and internal randomness before round $t$.
At round $t$, the environment selects a loss vector $\ell_t\in[0,1]^K$, and the learner selects an $\mathcal F_{t-1}$-measurable distribution $q_t$ over the arms.
The learner draws $I_t\sim q_t$, incurs loss $\ell_{t,I_t}$, and observes only this loss.

Our performance criterion is pseudo-regret against the best fixed arm in expectation.
The expectation includes learner and environment randomization.
Denote by $L_{T,k}\triangleq\E[\sum_{t=1}^T\ell_{t,k}]$ the expected cumulative loss of arm $k$ and by $L^\star\triangleq\min_{k\in[K]}L_{T,k}$ the smallest such loss.
For a fixed arm $k$, denote by $\Reg_{T,k}\triangleq\E[\sum_{t=1}^T\ell_{t,I_t}]-L_{T,k}$ the corresponding pseudo-regret.
The pseudo-regret against the best fixed arm is 
$$
 \Reg_T\triangleq\E\left[\sum_{t=1}^T\ell_{t,I_t}\right]-L^\star=\max_{k\in[K]}\Reg_{T,k}.
$$
For a deterministic loss sequence, $L^\star$ is the usual cumulative loss of the best arm in hindsight.

In the stochastic environment, the loss vectors are i.i.d. with means $\mu_i\triangleq\E[\ell_{t,i}]$ and a unique best arm $i^\star$.
Denote by $\Delta_i\triangleq\mu_i-\mu_{i^\star}>0$ the gap of arm $i\ne i^\star$.
The coordinates within a loss vector may be dependent.
Independence across rounds gives the identity
\begin{equation}
 \Reg_T=\sum_{i\ne i^\star}\Delta_i\E\left[\sum_{t=1}^Tq_{t,i}\right].
 \label{eq:selfbound}
\end{equation}
Our goal is a guarantee with the optimal leading first-order term $\order(\sqrt{KL^\star})$ in the adversarial environment, together with logarithmic gap-dependent regret in the stochastic environment.

\paragraph{Notation.}
Denote by $\simplex\triangleq\{x\in\mathbb R_+^K:\sum_i x_i=1\}$ the probability simplex and by $\operatorname{int}(\simplex)$ its relative interior, consisting of distributions with strictly positive coordinates.
Denote by $e_i$ the $i$-th standard basis vector, by $\one$ the all-ones vector, and by $\ind\{\cdot\}$ the indicator function.
For $u,v\in\mathbb R^K$, denote by $\ip{u}{v}\triangleq\sum_{i=1}^K u_iv_i$ the inner product.
Denote by $\Et[\cdot]\triangleq\E[\cdot\mid\mathcal F_{t-1},\ell_t]$ conditional expectation over the fresh draw at round $t$. For a differentiable convex function $\Phi$, denote by $D_\Phi(x,y)\triangleq\Phi(x)-\Phi(y)-\ip{\nabla\Phi(y)}{x-y}$ its Bregman divergence. 

\section{Minimax Optimal First-Order Regret via Fixed-Point Sampling}
\label{sec:lf}

In this section, we derive our minimax optimal first-order algorithm from FTRL with Tsallis entropy~\citep{zimmert2021tsallis}. In \pref{sec:loss-dep-rem}, we identify why the usual stability bound gives minimax regret but does not adapt to the learner's loss, and in \pref{sec:lf-sampling}, we derive a sampling distribution that controls estimation bias and stability together through a scalar fixed point, then choose the learning rate from observed cumulative loss.

\subsection{From FTRL with Tsallis Entropy to a Loss-Dependent Remainder}\label{sec:loss-dep-rem}

Our starting point is the $\frac{1}{2}$-Tsallis regularizer used by Tsallis-INF \citep{zimmert2021tsallis}. 
Denote by $\Psi(x)\triangleq-2\sum_i\sqrt{x_i}$ the negative $\frac{1}{2}$-Tsallis entropy regularizer. 
Fix a learning rate $\eta>0$, and denote by $\wh{L}_t\triangleq\sum_{s=1}^t\ellhat_s$ the cumulative loss estimate, with $\wh{L}_0=0$.
The FTRL distribution before round $t$ is 
$$
 x_t=\arg\min_{x\in\simplex}\left\{\ip{\wh{L}_{t-1}}x+\frac{1}{\eta}\Psi(x)\right\}.
$$
We first consider direct sampling $I_t\sim x_t$ with the unbiased estimate $\ellhat_{t,i}=\ell_{t,I_t}\ind\{I_t=i\}/x_{t,i}$.
The standard FTRL analysis \citep{zimmert2021tsallis} gives
\begin{equation}
 \sum_{t=1}^T\ip{x_t-e_k}{\ellhat_t}
 \le\frac{2\sqrt K}{\eta}+\eta\sum_{t=1}^T\sum_{i=1}^K x_{t,i}^{3/2}\ellhat_{t,i}^2.
 \label{eq:lf-ftrl-bound}
\end{equation}
Taking conditional expectations under direct sampling therefore yields
$$
 \Reg_{T,k}\le\order\left(\frac{\sqrt K}{\eta}
 +\eta\E\left[\sum_{t=1}^T\sum_{i=1}^K\sqrt{x_{t,i}}\,\ell_{t,i}^2\right]\right).
$$
Bounding the inner sum by $\sqrt{K}\cdot T$ gives $\order(\sqrt{KT})$ worst-case regret after picking $\eta=1/\sqrt{T}$.
A typical route to a first-order guarantee is to control the quadratic cost $\sum_{t=1}^T\sum_{i=1}^K\sqrt{x_{t,i}}\ell_{t,i}^2$ through the learner's loss $\sum_{t=1}^T\sum_{i=1}^Kx_{t,i}\ell_{t,i}$.
However, when $\ell_t=e_i$, the inner sum is $\sqrt{x_{t,i}}$, while the learner's expected loss is $x_{t,i}$. Their ratio is unbounded as $x_{t,i}$ approaches zero, so this term cannot be controlled uniformly by the learner's loss. 

To resolve this mismatch, we keep the FTRL update and change its sampling distribution from $x_t$ to $q_t$. We make two choices jointly: the loss estimator and the played distribution.
First, after drawing $I_t\sim q_t$, we use the implicit-exploration (IX) estimator \citep{kocak2014efficient,neu2015explore}, with positive offsets $\rho_{t,i}>0$ to be specified later:
\begin{equation}
 \ellhat_{t,i}
 =\frac{\ell_{t,I_t}\ind\{I_t=i\}}{q_{t,i}+\rho_{t,i}},
 \label{eq:lf-estimator}
\end{equation}
for every arm $i$.
The FTRL update still uses only $\wh{L}_t=\wh{L}_{t-1}+\ellhat_t$.
Under sampling from $q_t$, its conditional mean is $\Et[\ellhat_{t,i}]=q_{t,i}\ell_{t,i}/(q_{t,i}+\rho_{t,i})\le\ell_{t,i}$, so the estimator is biased downward. Its second moment is $\Et[\ellhat_{t,i}^2]=q_{t,i}\ell_{t,i}^2/(q_{t,i}+\rho_{t,i})^2$.
Second, we construct $q_t$ from $x_t$ so that the cost of this change and the quadratic term in \pref{eq:lf-ftrl-bound} can be controlled by the played loss.
Taking expectations in \pref{eq:lf-ftrl-bound} and using the fact that $\Et[\ellhat_{t,k}]\leq \ell_{t,k}$ for all $k\in[K]$ gives 
$$
 \E\left[\sum_{t=1}^T\left(\Et\left[\ip{x_t}{\ellhat_t}\right]-\eta\sum_{i=1}^K x_{t,i}^{3/2}\Et\left[\ellhat_{t,i}^2\right]\right)\right]
 \le L_{T,k}+\order\left(\frac{\sqrt K}{\eta}\right).
$$
Rearranging this inequality to bound $-L_{T,k}$ and adding $\E[\sum_t\ip{q_t}{\ell_t}]$ to both sides yields
\begin{equation}
 \Reg_{T,k}\le\order\left(\frac{\sqrt K}{\eta}\right)
 +\E\left[\sum_{t=1}^T\left(\ip{q_t}{\ell_t}-\Et\left[\ip{x_t}{\ellhat_t}\right]+\eta\sum_{i=1}^K x_{t,i}^{3/2}\Et\left[\ellhat_{t,i}^2\right]\right)\right].
 \label{eq:lf-design-remainder}
\end{equation}
Thus the two choices must control the combined remainder, including both estimation bias and the quadratic cost.
We aim to construct $q_t$ and an $\mathcal F_{t-1}$-measurable $Z_t\in(0,1)$ satisfying
\begin{align}
 \Et\left[\ip{x_t}{\ellhat_t}\right]-\eta\sum_{i=1}^K x_{t,i}^{3/2}\Et\left[\ellhat_{t,i}^2\right]
 &\ge Z_t\ip{q_t}{\ell_t}, \label{eq:lf-target}\\*
 1-Z_t&=\order(\eta\sqrt K). \nonumber
\end{align}
The first condition bounds the remainder in \pref{eq:lf-design-remainder} by $(1-Z_t)\ip{q_t}{\ell_t}$, and the second makes this quantity bounded by $\order(\eta\sqrt K)\ip{q_t}{\ell_t}$.
Using $\E[\sum_t\ip{q_t}{\ell_t}]=L_{T,k}+\Reg_{T,k}$ gives $\Reg_{T,k}\le\order(\sqrt K/\eta)+\order(\eta\sqrt K)(L_{T,k}+\Reg_{T,k})$. As long as $\eta\sqrt K$ is sufficiently small, rearranging the terms yields $\Reg_{T,k}=\order(\sqrt K/\eta+\eta\sqrt K L_{T,k})$. 
It remains to construct the offset $\rho_{t,i}$ and sampling distribution $q_t$ so that both conditions in \pref{eq:lf-target} hold without observing $\ell_t$.

\subsection{Deriving the Sampling Distribution}
\label{sec:lf-sampling}

We first leave $\rho_{t,i}$ unspecified and derive the conditions needed for \pref{eq:lf-target}.
Substituting the conditional moments of \pref{eq:lf-estimator} gives
\begin{align}
 \Et\left[\ip{x_t}{\ellhat_t}\right]-\eta\sum_{i=1}^K x_{t,i}^{3/2}\Et\left[\ellhat_{t,i}^2\right] \ge\sum_{i=1}^K q_{t,i}\ell_{t,i}
 \left(\frac{x_{t,i}}{q_{t,i}+\rho_{t,i}}-\frac{\eta x_{t,i}^{3/2}}{(q_{t,i}+\rho_{t,i})^2}\right),
 \label{eq:lf-net-loss}
\end{align}
where the inequality uses $\ell_{t,i}^2\le\ell_{t,i}$. Thus, to control the remainder by $(1-Z_t)\ip{q_t}{\ell_t}$ for every loss vector, it suffices to make
\begin{equation}
 \frac{x_{t,i}}{q_{t,i}+\rho_{t,i}}-\frac{\eta x_{t,i}^{3/2}}{(q_{t,i}+\rho_{t,i})^2}\ge Z_t,
 \label{eq:lf-coordinate-target}
\end{equation}
for every coordinate with $q_{t,i}>0$, while keeping $1-Z_t=\order(\eta\sqrt K)$.
This is the joint design requirement on the estimator offset and the sampling distribution.

To obtain a sampling rule $q_t$ from \pref{eq:lf-coordinate-target}, we aim to find a reciprocal lower bound on its coefficient.
Suppose that $q_{t,i}+\rho_{t,i}\ge2\eta\sqrt{x_{t,i}}$. The scalar inequality $1/d-a/d^2\ge1/(d+2a)$ for $d\ge2a>0$ therefore gives
\begin{equation}
 \frac{x_{t,i}}{q_{t,i}+\rho_{t,i}}-\frac{\eta x_{t,i}^{3/2}}{(q_{t,i}+\rho_{t,i})^2}
 \ge\frac{x_{t,i}}{q_{t,i}+\rho_{t,i}+2\eta\sqrt{x_{t,i}}}.
 \label{eq:lf-effective-loss}
\end{equation}
Therefore, an offset satisfying $\rho_{t,i}\ge2\eta\sqrt{x_{t,i}}$ makes this comparison valid uniformly over the sampling probabilities.
We impose this sufficient condition for now and choose the exact offset after checking the normalization cost.

To further match the right hand side of \pref{eq:lf-effective-loss} to $Z_t$, a positive coordinate of $q_t$ must satisfy $q_{t,i}=x_{t,i}/Z_t-\rho_{t,i}-2\eta\sqrt{x_{t,i}}$.
Denote by $[v]_+\triangleq\max\{v,0\}$ the positive part.
We therefore consider
\begin{equation}
 q_i(Z)=\left[\frac{x_{t,i}}Z-\rho_{t,i}-2\eta\sqrt{x_{t,i}}\right]_+,
 \label{eq:lf-generic-qroot}
\end{equation}
and determine $Z_t$ from $\sum_{i=1}^Kq_i(Z_t)=1$, with $q_{t,i}=q_i(Z_t)$. In order to ensure that $q_t\in\simplex$, it suffices to solve the following scalar fixed-point equation
\begin{equation}
 Z_t=\sum_{i=1}^K\left[x_{t,i}-Z_t\bigl(\rho_{t,i}+2\eta\sqrt{x_{t,i}}\bigr)\right]_+.
 \label{eq:lf-generic-fixedpoint}
\end{equation}
It couples the played probabilities to the estimator that uses those same probabilities in its denominator.
The fixed-point equation has a unique solution in $(0,1)$: $\sum_iq_i(Z)$ is continuous and strictly decreasing wherever positive, diverges as $Z$ approaches $0$, and is less than one at $Z=1$.

It remains to show that $1-Z_t$ satisfies the required bound in \pref{eq:lf-target}.
By \pref{eq:lf-generic-fixedpoint}, we have
\begin{align}
 1-Z_t&=\sum_{i=1}^K\left(x_{t,i}-\left[x_{t,i}-Z_t\bigl(\rho_{t,i}+2\eta\sqrt{x_{t,i}}\bigr)\right]_+\right) \nonumber\\*
 &=\sum_{i=1}^K\min\left\{x_{t,i},Z_t\bigl(\rho_{t,i}+2\eta\sqrt{x_{t,i}}\bigr)\right\}
 \le\sum_{i=1}^K\rho_{t,i}+2\eta\sqrt K,
 \label{eq:lf-generic-defect}
\end{align}
where the inequality uses $Z_t\le1$ and $\sum_i\sqrt{x_{t,i}}\le\sqrt K$.
It therefore suffices to choose offsets satisfying the sufficient coordinate condition in \pref{eq:lf-effective-loss} and $\sum_{i=1}^K\rho_{t,i}=\order(\eta\sqrt K)$ to control the normalization defect.
We now choose $\rho_{t,i}=2\eta\sqrt{x_{t,i}}$.
Its sum is at most $2\eta\sqrt K$, so it meets both requirements.
This choice simplifies \pref{eq:lf-generic-qroot} to $q_i(Z)=\left[\frac{x_{t,i}}Z-2\rho_{t,i}\right]_+$, and the fixed-point equation becomes $Z_t=\sum_{i=1}^K[x_{t,i}-2Z_t\rho_{t,i}]_+$. Recall that $q_{t,i}=q_i(Z_t)$ for all $i\in[K]$. If $q_{t,i}>0$, then $x_{t,i}=Z_t(q_{t,i}+2\rho_{t,i})$. Consequently, for every $i\in[K]$,
\begin{equation}
 \frac{x_{t,i}q_{t,i}}{q_{t,i}+2\rho_{t,i}}=Z_tq_{t,i},
 \label{eq:lf-matching}
\end{equation}
where the equality also holds trivially when $q_{t,i}=0$.
Combining \pref{eq:lf-net-loss}, \pref{eq:lf-effective-loss}, and \pref{eq:lf-matching} proves the first target in \pref{eq:lf-target}.
Moreover, 
$
 1-Z_t=\sum_{i=1}^K\min\left\{x_{t,i},2Z_t\rho_{t,i}\right\}
 \le2Z_t\sum_{i=1}^K\rho_{t,i}
 =\order(\eta\sqrt K)$,
which proves the second target in \pref{eq:lf-target} and also gives $Z_t\ge1/(1+4\eta\sqrt K)$. Picking $\eta=\Theta(\frac{1}{\sqrt{K+L^\star}})$ then gives the target leading first-order term $\order(\sqrt{KL^\star})$. To further achieve this target without knowing $L^\star$, we now use an adaptive learning rate choice
$\eta_t=\frac1{\sqrt{K+\sum_{s<t}\ell_{s,I_s}}}
$ to adapt the regularization scale to the learner's past observed losses. It turns out that this adaptive learning rate tuning also achieves the best-of-both-worlds guarantee as we will show in \pref{sec:bobw}.

Based on the above analysis, our formal algorithm is given in \pref{alg:lf}. At the beginning of round $t$, we compute $x_t$ by FTRL with Tsallis entropy and the adaptive learning rate (\pref{line:lf-ftrl}). 
The fixed-point calculation in \pref{line:lf-normalize} then determines the sampling distribution $q_t$ from $x_t$. Both $x_t$ and $Z_t$ involve solving one-dimensional monotone equations, in the simplex multiplier and in $Z$, respectively, so each round takes $\order(K\log(1/\epsilon))$ time by bisection to accuracy $\epsilon$. After sampling $I_t$ from $q_t$, we construct the IX-estimator $\ellhat_{t,i}$ with $\rho_{t,i}=2\eta_t\sqrt{x_{t,i}}$ and update the cumulative estimate. The following theorem gives the complete first-order guarantee.

\begin{samepage}\begin{theorem}
\phantomsection\label{thm:lf}
\pref{alg:lf}
guarantees that
$
 \Reg_{T,k}=\order(\sqrt{K L_{T,k}}+K),
$ for all $k\in[K]$.
\end{theorem}\end{samepage}

To our knowledge, this is the first explicit and efficient algorithm with a minimax optimal first-order guarantee $\order(\sqrt{K L^\star})$ for MAB, improving upon the $\order(\min\{\sqrt{KT},\sqrt{K L^\star}\log T\})$ guarantee in \citet{pogodin2020first}.

\begin{algorithm2e}[t]
\caption{Follow-the-Regularized-Leader with Tsallis Entropy and Fixed-Point Sampling}
\label{alg:lf}
\DontPrintSemicolon
\LinesNumbered
\KwIn{number of arms $K\ge2$, $\Psi(x)=-2\sum_{i=1}^K\sqrt{x_i}$.}
Initialize $\wh{L}_0=0$\;
\For{$t=1,2,\ldots$}{
Compute $x_t=\arg\min_{x\in\simplex}\left\{\ip{x}{\wh{L}_{t-1}}+\frac{1}{\eta_t}\Psi(x)\right\}$ with $\eta_t=\frac{1}{\sqrt{K+\sum_{\tau<t}\ell_{\tau, I_\tau}}}$\nllabel{line:lf-ftrl}\;
 Define $\rho_{t,i}=2\eta_t\sqrt{x_{t,i}}$ for every $i$ and find $Z_t\in(0,1)$ such that $\sum_i[x_{t,i}/Z_t-2\rho_{t,i}]_+=1$\nllabel{line:lf-normalize}\;
 Set $q_{t,i}=[x_{t,i}/Z_t-2\rho_{t,i}]_+$ for every $i\in[K]$\nllabel{line:lf-sampling}\;
 Draw $I_t\sim q_t$ and observe $\ell_{t,I_t}$\nllabel{line:lf-sample}\;
 Construct $\ellhat_{t,i}=\frac{\ell_{t,I_t}\ind\{I_t=i\}}{q_{t,i}+\rho_{t,i}}$ for every $i\in[K]$\;
 Set $\wh{L}_t=\wh{L}_{t-1}+\ellhat_t$\nllabel{line:lf-update}\;
}
\end{algorithm2e}

\subsection{Proof of \pref{thm:lf}}
In this subsection, we provide a proof of \pref{thm:lf}. 

\begin{proofof}{\pref{thm:lf}}
Via a standard analysis of FTRL with Tsallis entropy (included in \pref{lem:lf-ftrl}), we know that 
\begin{equation}
 \resizebox{0.9\linewidth}{!}{$\displaystyle
 \sum_{t=1}^T\ip{x_t-u}{\ellhat_t}
 \le\frac{\Psi(u)-\Psi(x_1)}{\eta_1}
 +\sum_{t=1}^T\left(\frac{1}{\eta_{t+1}}-\frac{1}{\eta_t}\right)(\Psi(u)-\Psi(x_{t+1}))
 +\sum_{t=1}^T\eta_t\sum_{i=1}^K x_{t,i}^{3/2}\ellhat_{t,i}^2,
 $}
 \label{eq:tsallis-decomposition}
\end{equation}
for every $u\in\simplex$, where we define the last-round learning rate as $\eta_{T+1}=1/\sqrt{K+\sum_{t=1}^T\ell_{t,I_t}}$.
For $u=e_k$, the first two terms on the right-hand side of \pref{eq:tsallis-decomposition} satisfy
\begin{align*}
 &\frac{\Psi(e_k)-\Psi(x_1)}{\eta_1}
 +\sum_{t=1}^T\left(\frac{1}{\eta_{t+1}}-\frac{1}{\eta_t}\right)
 \bigl(\Psi(e_k)-\Psi(x_{t+1})\bigr)\\
 &\le 2\sqrt K\left(\frac{1}{\eta_1}
 +\sum_{t=1}^T\left(\frac{1}{\eta_{t+1}}-\frac{1}{\eta_t}\right)\right)
 =\frac{2\sqrt K}{\eta_{T+1}},
\end{align*}
where the inequality uses $\Psi(e_k)-\Psi(x)=2(\sum_{i=1}^K\sqrt{x_i}-1)\le2\sqrt K$ for every $x\in\simplex$, and $\eta_{t+1}^{-1}-\eta_t^{-1}\ge0$ because the observed losses are nonnegative. Taking expectations and using the fact that $\Et[\ellhat_{t,k}]\le\ell_{t,k}$ therefore gives
$$
 \E\left[\sum_{t=1}^T\bigl(\ip{x_t}{\ellhat_t}-\eta_t\sum_{i=1}^K x_{t,i}^{3/2}\ellhat_{t,i}^2\bigr)\right]
 \le L_{T,k}+2\sqrt K\E\left[\frac{1}{\eta_{T+1}}\right].
$$
According to the definition of $\ellhat_{t,i}$, we have $\E_t[\ellhat_{t,i}]=\frac{q_{t,i}\ell_{t,i}}{q_{t,i}+\rho_{t,i}}$ and $\E_t[\ellhat_{t,i}^2]=\frac{q_{t,i}\ell_{t,i}^2}{(q_{t,i}+\rho_{t,i})^2}$, which gives
\begin{align*}
 \Et\left[\ip{x_t}{\ellhat_t}\right]-\eta_t\sum_{i=1}^K x_{t,i}^{3/2}\Et\left[\ellhat_{t,i}^2\right]
 &\ge\sum_{i=1}^Kq_{t,i}\ell_{t,i}\left(\frac{x_{t,i}}{q_{t,i}+\rho_{t,i}}-\frac{\eta_tx_{t,i}^{3/2}}{(q_{t,i}+\rho_{t,i})^2}\right)\\
 &\ge\sum_{i=1}^K\frac{x_{t,i}q_{t,i}\ell_{t,i}}{q_{t,i}+2\rho_{t,i}}
 =Z_t\ip{q_t}{\ell_t},
\end{align*}
where the first inequality uses $\ell_{t,i}^2\le\ell_{t,i}$, the second follows from the scalar inequality in \pref{eq:lf-effective-loss} with $q_{t,i}+\rho_{t,i}\ge2\eta_t\sqrt{x_{t,i}}$, and the equality follows from the sampling rule in \pref{line:lf-sampling}.
Since $\Et[\ell_{t,I_t}]=\ip{q_t}{\ell_t}$, these inequalities imply $\E[\sum_t Z_t\ell_{t,I_t}]\le L_{T,k}+2\sqrt K\E[\eta_{T+1}^{-1}]$.
Moreover, the fixed-point calculation in \pref{line:lf-normalize} gives $Z_t=\sum_{i=1}^K[x_{t,i}-2Z_t\rho_{t,i}]_+$. Therefore, we have
\begin{align}
 1-Z_t
 &=\sum_{i=1}^K\left(x_{t,i}-[x_{t,i}-2Z_t\rho_{t,i}]_+\right)
 =\sum_{i=1}^K\min\{x_{t,i},2Z_t\rho_{t,i}\}\nonumber\\
 &\le2Z_t\sum_{i=1}^K\rho_{t,i}
 =4\eta_tZ_t\sum_{i=1}^K\sqrt{x_{t,i}}
 \le4\eta_t\sqrt K,
 \label{eq:lf-adaptive-defect}
\end{align}
where the last inequality uses $Z_t\le1$ and $\sum_{i=1}^K\sqrt{x_{t,i}}\le\sqrt K$.
Adding the remaining learner's loss $\E\left[\sum_{t=1}^T(1-Z_t)\ell_{t,I_t}\right]$ and applying \pref{eq:lf-adaptive-defect} yields
\begin{align*}
 \Reg_{T,k}
 &\le2\sqrt K\E\left[\frac{1}{\eta_{T+1}}\right]
 +\E\left[\sum_{t=1}^T(1-Z_t)\ell_{t,I_t}\right]\\
 &\le2\sqrt K\E\left[\frac{1}{\eta_{T+1}}\right]
 +4\sqrt K\E\left[\sum_{t=1}^T\eta_t\ell_{t,I_t}\right].
\end{align*}
The learning rate in \pref{line:lf-ftrl} gives $\ell_{t,I_t}=(\eta_{t+1}^{-1}-\eta_t^{-1})(\eta_{t+1}^{-1}+\eta_t^{-1})$.
Consequently, we obtain that 
\begin{align*}
 \sum_{t=1}^T\eta_t\ell_{t,I_t}
 =\sum_{t=1}^T\left(\frac{1}{\eta_{t+1}}-\frac{1}{\eta_t}\right)\left(1+\frac{\eta_t}{\eta_{t+1}}\right)\le3\sum_{t=1}^T\left(\frac{1}{\eta_{t+1}}-\frac{1}{\eta_t}\right)
 =3\left(\frac{1}{\eta_{T+1}}-\frac{1}{\eta_1}\right),
\end{align*}
where the inequality uses $\eta_t/\eta_{t+1}=\sqrt{1+\eta_t^2\ell_{t,I_t}}\le2$, since $\eta_t\le1$ and $\ell_{t,I_t}\le1$. Substituting this bound into the regret inequality gives
\begin{equation}
 \Reg_{T,k}\le\order\left(\sqrt K\E\left[\sqrt{K+\sum_{t=1}^T\ell_{t,I_t}}\right]\right). \label{eq:lfmaster}
\end{equation}
Jensen's inequality and $\E[\sum_t\ell_{t,I_t}]=L_{T,k}+\Reg_{T,k}$ imply $\Reg_{T,k}\le C\sqrt{K(K+L_{T,k}+\Reg_{T,k})}$ for some universal constant $C>0$.
If the regret is negative, \pref{thm:lf} follows immediately. Otherwise, applying the AM-GM inequality gives $C\sqrt{K\Reg_{T,k}}\le\Reg_{T,k}/2+C^2K/2$, which further gives $\Reg_{T,k}\le\Reg_{T,k}/2+\order(\sqrt{K L_{T,k}}+K)$. 
Rearranging the terms gives $\Reg_{T,k}\le\order(\sqrt{K L_{T,k}}+K)$, proving \pref{thm:lf}. 
\end{proofof}

\section{Logarithmic Regret in the Stochastic Environment}
\label{sec:bobw}

In this section, we prove that the same algorithm in \pref{alg:lf} also achieves logarithmic gap-dependent regret in the stochastic environment. Together with \pref{thm:lf}, this gives a best-of-both-worlds guarantee with the optimal leading first-order term in the adversarial environment. Formally, we have the following theorem.

\begin{samepage}\begin{theorem}
\phantomsection\label{thm:bobw}
For $K\ge2$ and $T>2$, suppose that the loss vectors in $[0,1]^K$ are i.i.d. with a unique best arm $i^\star$. 
\pref{alg:lf} guarantees that
$$
 \Reg_T=\order\left(\sum_{\substack{i=1,i\ne i^\star}}^K\frac{\log T}{\Delta_i}\right).
$$
\end{theorem}\end{samepage}
To prove \pref{thm:bobw}, we follow the approach of \citet{zimmert2021tsallis} and derive a regret bound in terms of suboptimal-arm probabilities, which allows us to apply the self-bounding identity in \pref{eq:selfbound}. Starting from the FTRL bound in \pref{lem:lf-ftrl}, we account for the estimation bias and the difference between the FTRL update $x_t$ and the decision $q_t$ to obtain the following lemma.

\begin{lemma}
\phantomsection\label{lem:bobw-joint-cost}
Under the assumptions of \pref{thm:bobw}, \pref{alg:lf} guarantees that
\begin{align}
 \Reg_T
 &\le\frac{\Psi(e_{i^\star})-\Psi(x_1)}{\eta_1}+\E\left[\sum_{t=1}^T\left(\frac{1}{\eta_{t+1}}-\frac{1}{\eta_t}\right)(\Psi(e_{i^\star})-\Psi(x_{t+1}))\right]\nonumber\\
 &\quad+\sum_{t=1}^T\E\left[(1-Z_t)\sum_{\substack{i=1\\i\ne i^\star}}^K\Delta_iq_{t,i}\right]+14\sum_{t=1}^T\E\left[\eta_t\mu_{i^\star}\sum_{\substack{i=1\\i\ne i^\star}}^K\sqrt{x_{t,i}}\right].
 \label{eq:bobw-joint-cost}
\end{align}
\end{lemma}

We first show how \pref{lem:bobw-joint-cost} implies \pref{thm:bobw} and then sketch the proof of \pref{lem:bobw-joint-cost}. The detailed conditional-expectation calculations are given in \pref{app:bobw-joint-details}.

\begin{proofof}{\pref{thm:bobw}}
We first analyze the first two terms in \pref{eq:bobw-joint-cost}. Since $x_1$ is uniform and $\eta_1=1/\sqrt K$, direct calculation gives that
\begin{align*}
    \frac{\Psi(e_{i^\star})-\Psi(x_1)}{\eta_1} = \sqrt{K}\left(-2+2\sqrt{K}\right) \leq 2K.
\end{align*}
For the second term, we show in \pref{lem:bobw-coordinate} that $x_{t+1,i}\le6x_{t,i}$ for all $i\in[K]$. The learning rate in \pref{line:lf-ftrl} also satisfies $\eta_{t+1}\le\eta_t$ and
$
 \frac{1}{\eta_{t+1}}-\frac{1}{\eta_t}
 =\frac{\ell_{t,I_t}}{\frac{1}{\eta_{t+1}}+\frac{1}{\eta_t}}
 =\frac{\eta_t\eta_{t+1}\ell_{t,I_t}}{\eta_t+\eta_{t+1}}
 \le\frac{\eta_t\ell_{t,I_t}}{2}.
$
We therefore obtain
\begin{align*}
    &\left(\frac{1}{\eta_{t+1}}-\frac{1}{\eta_t}\right)\left(\Psi(e_{i^\star}) - \Psi(x_{t+1})\right)\\
    &=2\left(\frac{1}{\eta_{t+1}}-\frac{1}{\eta_t}\right)\left(\sum_{i=1}^K\sqrt{x_{t+1,i}}-1\right)\leq2\sqrt6\left(\frac{1}{\eta_{t+1}}-\frac{1}{\eta_t}\right)\sum_{\substack{i=1\\i\ne i^\star}}^K\sqrt{x_{t,i}}\le3\eta_t\ell_{t,I_t}\sum_{\substack{i=1\\i\ne i^\star}}^K\sqrt{x_{t,i}},
\end{align*}
where the first inequality uses $\sqrt{x_{t+1,i^\star}}\le1$ and \pref{lem:bobw-coordinate}, and the second uses $1/\eta_{t+1}-1/\eta_t\le\eta_t\ell_{t,I_t}/2$. Summing and taking expectations with $\E[\ell_{t,I_t}\mid\mathcal F_{t-1}]=\ip{q_t}\mu$ further gives
\begin{equation}
 \begin{aligned}
 &\frac{\Psi(e_{i^\star})-\Psi(x_1)}{\eta_1}+\E\left[\sum_{t=1}^T\left(\frac1{\eta_{t+1}}-\frac1{\eta_t}\right)(\Psi(e_{i^\star})-\Psi(x_{t+1}))\right]\\
 &\le2K+3\sum_{t=1}^T\E\left[\eta_t\ip{q_t}\mu\sum_{\substack{i=1\\i\ne i^\star}}^K\sqrt{x_{t,i}}\right].
 \end{aligned}
 \label{eq:bobw-local-penalty}
\end{equation}

For the third term of \pref{eq:bobw-joint-cost}, we consider two cases separately. When $\eta_t>1/(8\sqrt K)$, the learning-rate definition gives $\sum_{\tau<t}\ell_{\tau,I_\tau}<63K$. These rounds form an initial interval, and its final loss increment is at most one. Consequently, on every realization,
$
 \sum_{t=1}^T\ind\left\{\eta_t>\frac{1}{8\sqrt K}\right\}\ell_{t,I_t}\le63K+1.
$
The contribution of these rounds satisfies
\begin{align*}
 \sum_{t=1}^T\E\left[\ind\left\{\eta_t>\frac{1}{8\sqrt K}\right\}(1-Z_t)\sum_{\substack{i=1\\i\ne i^\star}}^K\Delta_iq_{t,i}\right] &\le\sum_{t=1}^T\E\left[\ind\left\{\eta_t>\frac{1}{8\sqrt K}\right\}\E[\ell_{t,I_t}\mid\mathcal F_{t-1}]\right]\\
 &=\E\left[\sum_{t=1}^T\ind\left\{\eta_t>\frac{1}{8\sqrt K}\right\}\ell_{t,I_t}\right]\le63K+1,
\end{align*}
where the first inequality uses $\sum_{i\ne i^\star}\Delta_iq_{t,i}\le\ip{q_t}\mu=\E[\ell_{t,I_t}\mid\mathcal F_{t-1}]$, the equality uses the tower property and $\mathcal F_{t-1}$-measurability of $\eta_t$, and the last inequality uses the preceding cumulative-loss bound. When $\eta_t\le1/(8\sqrt K)$, \pref{eq:lf-adaptive-defect} gives $1-Z_t\le4\eta_t\sqrt K\le1/2$. Thus
\begin{align*}
 \sum_{t=1}^T\E\left[\ind\left\{\eta_t\le\frac{1}{8\sqrt K}\right\}(1-Z_t)\sum_{\substack{i=1\\i\ne i^\star}}^K\Delta_iq_{t,i}\right]\le\frac12\sum_{t=1}^T\E\left[\ind\left\{\eta_t\le\frac{1}{8\sqrt K}\right\}\sum_{\substack{i=1\\i\ne i^\star}}^K\Delta_iq_{t,i}\right]
 \le\frac{\Reg_T}{2},
\end{align*}
where the last inequality follows from \pref{eq:selfbound} and the nonnegativity of the gaps. Adding the two cases yields
\begin{align}
 \sum_{t=1}^T\E\left[(1-Z_t)\sum_{\substack{i=1\\i\ne i^\star}}^K\Delta_iq_{t,i}\right]
 \le\frac{\Reg_T}{2}+63K+1.
 \label{eq:bobw-normalization-cost}
\end{align}
Substituting \pref{eq:bobw-local-penalty} and \pref{eq:bobw-normalization-cost} into \pref{eq:bobw-joint-cost}, and using $\mu_{i^\star}\le\ip{q_t}\mu$, gives
\begin{align}
 \Reg_T\le\frac{\Reg_T}{2}+\order(K)+17\sum_{t=1}^T\E\left[\eta_t\ip{q_t}\mu\sum_{\substack{i=1\\i\ne i^\star}}^K\sqrt{x_{t,i}}\right].
 \label{eq:bobw-x-regret}
\end{align}
To express this bound through the sampling probability $q_t$, note that according to the definition of $q_t$ in \pref{line:lf-sampling}, we have
\begin{align*}
 q_{t,i}
 =\left[\frac{x_{t,i}}{Z_t}-4\eta_t\sqrt{x_{t,i}}\right]_+
 \ge\frac{x_{t,i}}{Z_t}-4\eta_t\sqrt{x_{t,i}}
 \ge x_{t,i}-4\eta_t\sqrt{x_{t,i}},
\end{align*}
where the second inequality uses $Z_t\le1$. Completing the square gives $(\sqrt{x_{t,i}}-2\eta_t)^2\le q_{t,i}+4\eta_t^2$, and hence $\sqrt{x_{t,i}}\le2\eta_t+\sqrt{q_{t,i}+4\eta_t^2}\le\sqrt{q_{t,i}}+4\eta_t.$
Substituting this into \pref{eq:bobw-x-regret} and rearranging the terms yields
\begin{align}\label{eq:inter}
 \Reg_T
&\le\order(K)+34\sum_{t=1}^T\E\left[\eta_t\ip{q_t}\mu\sum_{\substack{i=1\\i\ne i^\star}}^K\sqrt{q_{t,i}}\right]
 +4\cdot34(K-1)\sum_{t=1}^T\E\left[\eta_t^2\ip{q_t}\mu\right].
\end{align}
To bound the last term in \pref{eq:inter}, direct calculation shows that
\begin{align*}
 \sum_{t=1}^T\eta_t^2\ell_{t,I_t}
 &\le2\sum_{t=1}^T\log\left(\frac{\eta_t^2}{\eta_{t+1}^2}\right)
 =2\log\left(\frac{\eta_1^2}{\eta_{T+1}^2}\right) \le2\log\left(1+\frac{T}{K}\right)\le2\log T,
\end{align*}
where the first inequality applies $u\le2\log(1+u)$ with $u=\eta_t^2\ell_{t,I_t}\le1/2$ and $1+u=\eta_t^2/\eta_{t+1}^2$, and the last two inequalities use $\sum_{t=1}^T\ell_{t,I_t}\le T$, $K\ge2$, and $T>2$.
Since $\eta_t$ is $\mathcal F_{t-1}$-measurable, we have
\begin{align}
 \sum_{t=1}^T\E\left[\eta_t^2\ip{q_t}\mu\right]
=\sum_{t=1}^T\E\left[\eta_t^2\E[\ell_{t,I_t}\mid\mathcal F_{t-1}]\right]=\E\left[\sum_{t=1}^T\eta_t^2\ell_{t,I_t}\right]\le2\log T.
 \label{eq:bobw-square-rate}
\end{align}
Applying \pref{eq:bobw-square-rate} to \pref{eq:inter} gives
\begin{align}
 \Reg_T\le\order(K\log T)+34\sum_{t=1}^T\sum_{\substack{i=1\\i\ne i^\star}}^K\E[\eta_t\ip{q_t}\mu\sqrt{q_{t,i}}].
 \label{eq:bobw-localized-regret}
\end{align}
For each $i\ne i^\star$, the AM-GM inequality gives
$
 34\eta_t\ip{q_t}\mu\sqrt{q_{t,i}}
 \le\frac{\Delta_iq_{t,i}}{2}+\frac{34^2\eta_t^2\ip{q_t}\mu^2}{2\Delta_i}
 \le\frac{\Delta_iq_{t,i}}{2}+\frac{34^2\eta_t^2\ip{q_t}\mu}{2\Delta_i}.
$
Summing these bounds and applying \pref{eq:selfbound} and \pref{eq:bobw-square-rate} yields
\begin{align*}
 34\sum_{t=1}^T\sum_{\substack{i=1\\i\ne i^\star}}^K\E[\eta_t\ip{q_t}\mu\sqrt{q_{t,i}}]
 \le\frac{\Reg_T}{2}+\frac{34^2}{2}\sum_{\substack{i=1\\i\ne i^\star}}^K\frac{1}{\Delta_i}\sum_{t=1}^T\E[\eta_t^2\ip{q_t}\mu]\le\frac{\Reg_T}{2}+34^2\sum_{\substack{i=1\\i\ne i^\star}}^K\frac{\log T}{\Delta_i}.
\end{align*}
Since $0<\Delta_i\le1$ and $K\ge2$, $\order(K\log T)$ is bounded by $\order(\sum_{i\neq i^\star}\frac{\log T}{\Delta_i})$. Substituting the preceding inequality into \pref{eq:bobw-localized-regret} therefore gives $\Reg_T\le\Reg_T/2+\order(\sum_{i\ne i^\star}\frac{\log T}{\Delta_i})$. Rearranging the terms proves \pref{thm:bobw}.
\end{proofof}

\subsection{Proof Sketch of \pref{lem:bobw-joint-cost}}\label{sec:lemma-bobw-sketch}
In this section, we provide a proof sketch of \pref{lem:bobw-joint-cost}, with the complete proof deferred to \pref{app:bobw-joint-details}. By adapting the FTRL analysis with Tsallis entropy to our loss estimator, we show in \pref{lem:lf-ftrl} that
\begin{align}
    \sum_{t=1}^T\ip{x_t-e_{i^\star}}{\ellhat_t}
 &\le\frac{\Psi(e_{i^\star})-\Psi(x_1)}{\eta_1}
 +\sum_{t=1}^T\left(\frac{1}{\eta_{t+1}}-\frac{1}{\eta_t}\right)(\Psi(e_{i^\star})-\Psi(x_{t+1}))\nonumber\\
 &\qquad+\sum_{t=1}^T\eta_t\min\left\{x_{t,I_t}^{3/2},8\sum_{i=1,i\neq I_t}^Kx_{t,i}^{3/2}\right\}\ellhat_{t,I_t}^2.\label{eq:bobw-inter}
\end{align}
Taking the conditional expectation of the final summand, with some calculation, one can obtain that
\begin{align}
 \E\left[\eta_t\min\left\{x_{t,I_t}^{3/2},8\sum_{i=1,i\ne I_t}^Kx_{t,i}^{3/2}\right\}\ellhat_{t,I_t}^2\mid\mathcal F_{t-1}\right]\le Z_t\sum_{i=1,i\ne i^\star}^K\frac{q_{t,i}\rho_{t,i}\Delta_i}{q_{t,i}+\rho_{t,i}}+10\eta_t\mu_{i^\star}\sum_{i=1,i\ne i^\star}^K\sqrt{x_{t,i}}.\label{eq:bobw-inter-2}
\end{align}

In order to transform the left hand side of \pref{eq:bobw-inter} to $\Reg_T$, we need to account for the estimator and the change from $x_t$ to $q_t$. Using \pref{eq:lf-matching} and $\mu_i=\mu_{i^\star}+\Delta_i$, one can show that
\begin{align}
 &\E\left[\ip{q_t-e_{i^\star}}{\ell_t}-\ip{x_t-e_{i^\star}}{\ellhat_t}\mid\mathcal F_{t-1}\right] \nonumber\\ 
 &=(1-Z_t)\sum_{\substack{i=1\\i\ne i^\star}}^K\Delta_iq_{t,i}+\mu_{i^\star}\left(\sum_{\substack{i=1\\i\ne i^\star}}^K\frac{x_{t,i}\rho_{t,i}}{q_{t,i}+\rho_{t,i}}-\frac{(1-x_{t,i^\star})\rho_{t,i^\star}}{q_{t,i^\star}+\rho_{t,i^\star}}\right)-Z_t\sum_{\substack{i=1\\i\ne i^\star}}^K\frac{q_{t,i}\rho_{t,i}\Delta_i}{q_{t,i}+\rho_{t,i}}.\label{eq:bobw-inter-3}
\end{align}

Combining \pref{eq:bobw-inter}, \pref{eq:bobw-inter-2}, and \pref{eq:bobw-inter-3}, noticing that the negative term \pref{eq:bobw-inter-3} cancels the first term in \pref{eq:bobw-inter-2}, we can bound $\Reg_T$ as follows:

\begin{align*}
    \Reg_T &= \sum_{t=1}^T\E\left[\ip{x_t-e_{i^\star}}{\ellhat_t}\right] + \sum_{t=1}^T\E\left[\ip{q_t-e_{i^\star}}{\ell_t}-\ip{x_t-e_{i^\star}}{\ellhat_t}\right]\\
    &\le \frac{\Psi(e_{i^\star})-\Psi(x_1)}{\eta_1}+\E\left[\sum_{t=1}^T\left(\frac{1}{\eta_{t+1}}-\frac{1}{\eta_t}\right)(\Psi(e_{i^\star})-\Psi(x_{t+1}))\right]\\
    &\qquad+10\sum_{t=1}^T\E\left[\eta_t\mu_{i^\star}\sum_{i=1,i\ne i^\star}^K\sqrt{x_{t,i}}\right]\\
    &\qquad+\sum_{t=1}^T\E\left[(1-Z_t)\sum_{\substack{i=1\\i\ne i^\star}}^K\Delta_iq_{t,i}+\mu_{i^\star}\left(\sum_{\substack{i=1\\i\ne i^\star}}^K\frac{x_{t,i}\rho_{t,i}}{q_{t,i}+\rho_{t,i}}-\frac{(1-x_{t,i^\star})\rho_{t,i^\star}}{q_{t,i^\star}+\rho_{t,i^\star}}\right)\right],
\end{align*}
Moreover, as $q_{t,i}=[\frac{x_{t,i}}{Z_t}-2\rho_{t,i}]_+\geq \frac{x_{t,i}}{Z_t}-2\rho_{t,i} $, we know that $$\frac{x_{t,i}\rho_{t,i}}{q_{t,i}+\rho_{t,i}}\le \frac{Z_t(q_{t,i}+2\rho_{t,i})\rho_{t,i}}{q_{t,i}+\rho_{t,i}}\leq  2Z_t\rho_{t,i}\le4\eta_t\sqrt{x_{t,i}},$$ since $Z_t\le1$ and $\rho_{t,i}=2\eta_t\sqrt{x_{t,i}}$. Applying the preceding bound and dropping the non-positive terms lead to \pref{eq:bobw-joint-cost}.

\section{Conclusion}

In this paper, we design an algorithm with the minimax-optimal leading first-order regret term $\order(\sqrt{K L^\star})$ in the adversarial environment. To our knowledge, this is the first efficient algorithm to achieve this minimax optimal first-order scale without a logarithmic factor. Moreover, the same algorithm simultaneously achieves $\order(\sum_{i\neq i^\star}\frac{\log T}{\Delta_i})$ regret in the stochastic environment with a unique best arm, giving a best-of-both-worlds guarantee. Technically, we build on FTRL with Tsallis entropy and use a scalar fixed point to couple sampling with implicit exploration. This coupling controls estimation bias and stability jointly, while a learning rate based on the learner's observed cumulative loss adapts to both environments. Future work includes extending this approach to richer partial-feedback models and obtaining high-probability guarantees.

\section*{AI Disclosure}
The authors used GPT-6 Astra to assist with exploring proof strategies and checking mathematical derivations. The authors also use Opus 5.5 and GPT-6 Astra for identifying relevant literature and polishing the manuscript. The authors directed the research, critically examined and revised the proposed arguments, and iteratively refined the algorithm, proofs, and exposition. All decisions about the results and the presentation were made by the authors, who are responsible for the correctness, originality, and attribution of all content.

\bibliographystyle{plainnat}
\bibliography{references}

\appendix

\section{Omitted Details in \pref{sec:lf}}
\label{app:lf-proof}

In this section, we prove the FTRL bound used in \pref{sec:lf} and \pref{sec:bobw}, which follows the analysis of FTRL with Tsallis entropy with an adaptation to our IX-estimator.
\begin{lemma}
\phantomsection\label{lem:lf-ftrl}
\pref{alg:lf} guarantees that
\begin{align*}
 \sum_{t=1}^T\ip{x_t-u}{\ellhat_t}
 &\le\frac{\Psi(u)-\Psi(x_1)}{\eta_1}
 +\sum_{t=1}^T\left(\frac{1}{\eta_{t+1}}-\frac{1}{\eta_t}\right)(\Psi(u)-\Psi(x_{t+1}))\\
 &\quad+\sum_{t=1}^T\eta_t\min\left\{x_{t,I_t}^{3/2},8\sum_{\substack{i=1\\i\ne I_t}}^Kx_{t,i}^{3/2}\right\}\ellhat_{t,I_t}^2,
\end{align*}
for every $u\in\simplex$.
\end{lemma}

\begin{proof}
Denote by
$$
 \Phi_t\triangleq\min_{z\in\simplex}\left\{\ip{z}{\wh{L}_{t-1}}+\frac{\Psi(z)}{\eta_t}\right\}
 =\ip{x_t}{\wh{L}_{t-1}}+\frac{\Psi(x_t)}{\eta_t}
$$
the FTRL potential at round $t$. The optimality condition for $x_t$ gives $\ip{\wh{L}_{t-1}+\nabla\Psi(x_t)/\eta_t}{z-x_t}=0$ for every $z\in\simplex$. Expanding the Bregman divergence and using $\wh{L}_t=\wh{L}_{t-1}+\ellhat_t$ therefore gives the identity
\begin{align*}
 \Phi_{t+1}-\Phi_t
 =\ip{x_{t+1}}{\ellhat_t}+D_{\eta_t^{-1}\Psi}(x_{t+1},x_t)+\left(\frac1{\eta_{t+1}}-\frac1{\eta_t}\right)\Psi(x_{t+1}).
\end{align*}
Rearranging the terms and bounding the expression at $x_{t+1}$ by its maximum over the simplex yields
\begin{align*}
 \ip{x_t}{\ellhat_t}
 \le\Phi_{t+1}-\Phi_t-\left(\frac1{\eta_{t+1}}-\frac1{\eta_t}\right)\Psi(x_{t+1})+\max_{z\in\simplex}\left\{\ip{x_t-z}{\ellhat_t}-D_{\eta_t^{-1}\Psi}(z,x_t)\right\}.
\end{align*}
Since $\Phi_1=\Psi(x_1)/\eta_1$ and $\Phi_{T+1}\le\ip{u}{\wh{L}_T}+\Psi(u)/\eta_{T+1}$, summing this inequality and subtracting $\ip{u}{\wh{L}_T}$ gives
\begin{align}
 \sum_{t=1}^T\ip{x_t-u}{\ellhat_t}
 &\le\frac{\Psi(u)}{\eta_{T+1}}-\frac{\Psi(x_1)}{\eta_1}-\sum_{t=1}^T\left(\frac1{\eta_{t+1}}-\frac1{\eta_t}\right)\Psi(x_{t+1})\nonumber\\
 &\quad+\sum_{t=1}^T\max_{z\in\simplex}\left\{\ip{x_t-z}{\ellhat_t}-D_{\eta_t^{-1}\Psi}(z,x_t)\right\}.\label{eq:ftrl-tsallis-app}
\end{align}
Using $1/\eta_{T+1}=1/\eta_1+\sum_{t=1}^T(1/\eta_{t+1}-1/\eta_t)$ gives the first two terms in the lemma.

It remains to bound the maximization term in \pref{eq:ftrl-tsallis-app}. Fix a round $t$. For $z_{I_t}\ge x_{t,I_t}$, this term is nonpositive. For $z_{I_t}<x_{t,I_t}$, we have
\begin{align*}
 D_{\eta_t^{-1}\Psi}(z,x_t)
 &\ge\frac1{2\eta_t}\int_{z_{I_t}}^{x_{t,I_t}}\frac{s-z_{I_t}}{s^{3/2}}\,ds
 \ge\frac{(x_{t,I_t}-z_{I_t})^2}{4\eta_tx_{t,I_t}^{3/2}},
\end{align*}
where the first inequality uses separability of $\Psi$ and nonnegativity of the other coordinate divergences, and the second uses $s\le x_{t,I_t}$ throughout the integral. Completing the square gives
\begin{align}\label{eq:first-bound}
 \ip{x_t-z}{\ellhat_t}-D_{\eta_t^{-1}\Psi}(z,x_t)
 &\le\ellhat_{t,I_t}(x_{t,I_t}-z_{I_t})-\frac{(x_{t,I_t}-z_{I_t})^2}{4\eta_tx_{t,I_t}^{3/2}}
 \le\eta_tx_{t,I_t}^{3/2}\ellhat_{t,I_t}^2.
\end{align}
For the second bound, denote by $z_t\triangleq\argmax_{z\in\simplex}\left\{\ip{x_t-z}{\ellhat_t}-D_{\eta_t^{-1}\Psi}(z,x_t)\right\}$ the maximizer. Direct calculation shows that
\begin{align}
 \max_{z\in\simplex}\left\{\ip{x_t-z}{\ellhat_t}-D_{\eta_t^{-1}\Psi}(z,x_t)\right\}
 &\le\ip{x_t-z_t}{\ellhat_t}
 =\ellhat_{t,I_t}(x_{t,I_t}-z_{t,I_t})=\ellhat_{t,I_t}\sum_{\substack{i=1\\i\ne I_t}}^K(z_{t,i}-x_{t,i}),
 \label{eq:bobw-mass-transfer}
\end{align}
where the inequality uses nonnegativity of the Bregman divergence, the first equality uses $\ellhat_{t,i}=0$ for $i\ne I_t$, and the second uses $\sum_i x_{t,i}=\sum_i z_{t,i}=1$. Consider two cases. If $I_t\neq \argmax_{i\in[K]}x_{t,i}$, then the first bound \pref{eq:first-bound} implies the second because $x_{t,I_t}^{3/2}\le\sum_{i\ne I_t}x_{t,i}^{3/2}$. Suppose now that $I_t= \argmax_{i\in[K]}x_{t,i}$. The KKT conditions for $z_t$ give
\begin{align*}
 -\ellhat_{t,i}-\frac{\nabla_i\Psi(z_t)-\nabla_i\Psi(x_t)}{\eta_t}+\lambda&=0,\\
 \frac{1}{\eta_t\sqrt{x_{t,i}}}-\frac{1}{\eta_t\sqrt{z_{t,i}}}
 &=\lambda-\ellhat_{t,I_t}\ind\{i=I_t\},
\end{align*}
for every $i\in[K]$, where $\lambda$ is the simplex multiplier and the second equality uses $\nabla_i\Psi(x)=-1/\sqrt{x_i}$. The simplex constraint gives $0\le\lambda\le\ellhat_{t,I_t}$: outside this interval, every coordinate would strictly decrease or strictly increase. Consequently,
\begin{align*}
 0\le\eta_t\lambda\sqrt{x_{t,i}}
 &\le\frac{\eta_t\sqrt{x_{t,i}}}{q_{t,I_t}+2\eta_t\sqrt{x_{t,I_t}}}
 \le\frac{\sqrt{x_{t,i}}}{2\sqrt{x_{t,I_t}}}\le\frac12,
\end{align*}
for $i\ne I_t$, where the first upper bound uses $\lambda\le\ellhat_{t,I_t}$, $\ell_{t,I_t}\le1$, and $\rho_{t,I_t}=2\eta_t\sqrt{x_{t,I_t}}$, the second uses $q_{t,I_t}\ge0$, and the last uses maximality of $x_{t,I_t}$. Solving the KKT equation for these coordinates therefore gives
\begin{align*}
 z_{t,i}-x_{t,i}
 =\frac{x_{t,i}}{(1-\eta_t\lambda\sqrt{x_{t,i}})^2}-x_{t,i}=x_{t,i}\frac{2\eta_t\lambda\sqrt{x_{t,i}}-(\eta_t\lambda\sqrt{x_{t,i}})^2}{(1-\eta_t\lambda\sqrt{x_{t,i}})^2}
 \le8\eta_t\lambda x_{t,i}^{3/2},
\end{align*}
where the inequality uses $0\le\eta_t\lambda\sqrt{x_{t,i}}\le1/2$. Substituting into \pref{eq:bobw-mass-transfer} and using $\lambda\le\ellhat_{t,I_t}$ gives
\begin{align*}
 \max_{z\in\simplex}\left\{\ip{x_t-z}{\ellhat_t}-D_{\eta_t^{-1}\Psi}(z,x_t)\right\}
 &\le8\eta_t\lambda\ellhat_{t,I_t}\sum_{\substack{i=1\\i\ne I_t}}^Kx_{t,i}^{3/2}
 \le8\eta_t\ellhat_{t,I_t}^2\sum_{\substack{i=1\\i\ne I_t}}^Kx_{t,i}^{3/2}.
\end{align*}
Combining the two bounds yields
\begin{align}
 \max_{z\in\simplex}\left\{\ip{x_t-z}{\ellhat_t}-D_{\eta_t^{-1}\Psi}(z,x_t)\right\}
 \le\eta_t\min\left\{x_{t,I_t}^{3/2},8\sum_{\substack{i=1\\i\ne I_t}}^Kx_{t,i}^{3/2}\right\}\ellhat_{t,I_t}^2.
 \label{eq:bobw-update-cost}
\end{align}
Substituting \pref{eq:bobw-update-cost} into the FTRL decomposition proves the lemma.
\end{proof}

\section{Omitted Details in \pref{sec:bobw}}
\label{app:bobw-proof}
\label{app:bridge}
\label{app:cc-proof}
\label{app:transport-proof}

In this section, we prove \pref{lem:bobw-coordinate}, which compares consecutive iterates after the loss estimate and learning rate are updated. We then provide the complete proof for \pref{lem:bobw-joint-cost}, which gives the detailed calculation of the conditional-expectation used in \pref{sec:lemma-bobw-sketch}.

\subsection{Multiplicative Stability}
\label{sec:cc-proof}

The following lemma proves the multiplicative stability of the consecutive iterates after both the loss estimate and the learning rate are updated. 

\begin{lemma}
\phantomsection\label{lem:bobw-coordinate}
\pref{alg:lf} satisfies that
$
 x_{t+1,i}\le6x_{t,i},
$
for every $t\in[T]$ and $i\in[K]$.
\end{lemma}

\begin{proof}
Fix a round $t$. Denote by $\lambda_t$ and $\lambda_{t+1}$ the simplex multipliers for the two consecutive FTRL updates. Their KKT equations give
\begin{align*}
 \wh{L}_{t-1,i}+\lambda_t&=\frac{1}{\eta_t\sqrt{x_{t,i}}},\\
 \wh{L}_{t,i}+\lambda_{t+1}&=\frac{1}{\eta_{t+1}\sqrt{x_{t+1,i}}},
\end{align*}
for every $i\in[K]$. Subtracting the first equation from the second and using $\wh{L}_t=\wh{L}_{t-1}+\widehat\ell_t$, we obtain that
\begin{align*}
 \frac{1}{\eta_{t+1}\sqrt{x_{t+1,i}}}
 =\frac{1}{\eta_t\sqrt{x_{t,i}}}
   +\widehat\ell_{t,i}+\lambda_{t+1}-\lambda_t=\frac{1+\eta_t\sqrt{x_{t,i}}
   (\widehat\ell_{t,i}+\lambda_{t+1}-\lambda_t)}
   {\eta_t\sqrt{x_{t,i}}},
\end{align*}
where the numerator $1+\eta_t\sqrt{x_{t,i}}
   (\widehat\ell_{t,i}+\lambda_{t+1}-\lambda_t)$ is positive as the left-hand side is positive. Taking reciprocals and squaring therefore gives
$$
 x_{t+1,i}
 =\frac{\left(\frac{\eta_t}{\eta_{t+1}}\right)^2x_{t,i}}
 {\left(1+\eta_t\sqrt{x_{t,i}}
 (\widehat\ell_{t,i}+\lambda_{t+1}-\lambda_t)\right)^2}.
$$
The learning-rate update implies $(\eta_t/\eta_{t+1})^2=1+\eta_t^2\ell_{t,I_t}\le3/2$, since $\ell_{t,I_t}\le1$, $\eta_t^2\le1/K$, and $K\ge2$. If $\lambda_{t+1}\ge\lambda_t$, the denominator in the coordinate formula is at least one because $\widehat\ell_{t,i}\ge0$, giving $x_{t+1,i}\le3x_{t,i}/2$ for every $i$.

It remains to consider $\lambda_{t+1}<\lambda_t$. For every $i\ne I_t$, we have $\widehat\ell_{t,i}=0$, so the coordinate formula and $\eta_{t+1}\le\eta_t$ imply $x_{t+1,i}>x_{t,i}$. Since both iterates sum to one, this also gives $x_{t+1,I_t}\le x_{t,I_t}$. For $i\ne I_t$ with $x_{t,i}\ge x_{t,I_t}$, we know that
\begin{align*}
 x_{t+1,i}-x_{t,i}
 \le\sum_{\substack{j=1\\j\ne I_t}}^K
       (x_{t+1,j}-x_{t,j})=x_{t,I_t}-x_{t+1,I_t}
 \le x_{t,I_t}
 \le x_{t,i},
\end{align*}
where the first inequality uses the nonnegativity of every summand, and the equality uses the simplex constraints. Thus, $x_{t+1,i}\le2x_{t,i}$ for these coordinates.

To bound the remaining coordinates, we first show that $\lambda_{t+1}\ge\lambda_t-\widehat\ell_{t,I_t}$. Otherwise, $\widehat\ell_{t,i}+\lambda_{t+1}-\lambda_t<0$ for every $i\in[K]$, so the coordinate formula and $\eta_{t+1}\le\eta_t$ would imply $x_{t+1,i}>x_{t,i}$ for every $i$, contradicting the simplex constraints. Consequently, for $i\ne I_t$ with $x_{t,i}<x_{t,I_t}$, we have
\begin{align*}
 \eta_t\sqrt{x_{t,i}}(\lambda_t-\lambda_{t+1})
\le\eta_t\sqrt{x_{t,i}}\widehat\ell_{t,I_t}\le\frac{\eta_t\sqrt{x_{t,i}}}{\rho_{t,I_t}}
 =\frac{\sqrt{x_{t,i}}}{2\sqrt{x_{t,I_t}}}
 <\frac12,
\end{align*}
where the second inequality uses $\widehat\ell_{t,I_t}=\ell_{t,I_t}/(q_{t,I_t}+\rho_{t,I_t})\le1/\rho_{t,I_t}$, and the equality uses $\rho_{t,I_t}=2\eta_t\sqrt{x_{t,I_t}}$. Substituting this bound into the coordinate formula gives
$$
 x_{t+1,i}
 =\frac{\left(\frac{\eta_t}{\eta_{t+1}}\right)^2x_{t,i}}
 {\left(1-\eta_t\sqrt{x_{t,i}}(\lambda_t-\lambda_{t+1})\right)^2}
 \le4\left(\frac{\eta_t}{\eta_{t+1}}\right)^2x_{t,i}
 \le6x_{t,i}.
$$
Combining the bounds for all coordinates proves the lemma.
\end{proof}

\subsection{Proof of \pref{lem:bobw-joint-cost}}
\label{app:bobw-joint-details}

\begin{proof}
We first apply \pref{lem:lf-ftrl} with $u=e_{i^\star}$, which gives \pref{eq:bobw-inter}:
\begin{align*}
 \sum_{t=1}^T\ip{x_t-e_{i^\star}}{\ellhat_t}
 &\le\frac{\Psi(e_{i^\star})-\Psi(x_1)}{\eta_1}
 +\sum_{t=1}^T\left(\frac{1}{\eta_{t+1}}-\frac{1}{\eta_t}\right)(\Psi(e_{i^\star})-\Psi(x_{t+1}))\\
 &\quad+\sum_{t=1}^T\eta_t\min\left\{x_{t,I_t}^{3/2},8\sum_{\substack{i=1\\i\ne I_t}}^Kx_{t,i}^{3/2}\right\}\ellhat_{t,I_t}^2.
\end{align*}
We then prove the conditional bound in \pref{eq:bobw-inter-2} for the last term in \pref{eq:bobw-inter}, then derive the difference between true and estimated regret in \pref{eq:bobw-inter-3}, and finally combine these bounds to obtain \pref{eq:bobw-joint-cost}.

All conditional expectations below are conditioned on $\mathcal F_{t-1}$. Direct calculation shows that
\begin{align*}
 \E[\ellhat_{t,i}\mid\mathcal F_{t-1}]
 &=\frac{q_{t,i}\mu_i}{q_{t,i}+\rho_{t,i}},\\
 \E[\ellhat_{t,i}^2\mid\mathcal F_{t-1}]
 &=\frac{q_{t,i}\E[\ell_{t,i}^2]}{(q_{t,i}+\rho_{t,i})^2}
 \le\frac{q_{t,i}\mu_i}{(q_{t,i}+\rho_{t,i})^2},
\end{align*}
where the equalities use independence of the current loss vector from the history and the fresh action. The sampling rule in \pref{line:lf-sampling} also gives
\begin{align*}
 \frac{x_{t,i}}{q_{t,i}+\rho_{t,i}}
 &\le Z_t\frac{q_{t,i}+2\rho_{t,i}}{q_{t,i}+\rho_{t,i}}
 \le2Z_t,
\end{align*}
where the first inequality follows from $q_{t,i}=[x_{t,i}/Z_t-2\rho_{t,i}]_+\ge x_{t,i}/Z_t-2\rho_{t,i}$. Hence, for every suboptimal arm $i\ne i^\star$, we have
\begin{align*}
 \eta_tx_{t,i}^{3/2}\E[\ellhat_{t,i}^2\mid\mathcal F_{t-1}]
 &\le\frac{\eta_tx_{t,i}^{3/2}q_{t,i}\mu_i}{(q_{t,i}+\rho_{t,i})^2}\\
 &=\frac{\rho_{t,i}}{2}\frac{x_{t,i}}{q_{t,i}+\rho_{t,i}}\frac{q_{t,i}\mu_i}{q_{t,i}+\rho_{t,i}}\\
 &\le Z_t\frac{q_{t,i}\rho_{t,i}\mu_i}{q_{t,i}+\rho_{t,i}}\\
 &=Z_t\frac{q_{t,i}\rho_{t,i}\Delta_i}{q_{t,i}+\rho_{t,i}}+Z_t\frac{q_{t,i}\rho_{t,i}\mu_{i^\star}}{q_{t,i}+\rho_{t,i}}\\
 &\le Z_t\frac{q_{t,i}\rho_{t,i}\Delta_i}{q_{t,i}+\rho_{t,i}}+2\eta_t\mu_{i^\star}\sqrt{x_{t,i}},
\end{align*}
where the first equality uses $\rho_{t,i}=2\eta_t\sqrt{x_{t,i}}$, the second inequality uses $x_{t,i}/(q_{t,i}+\rho_{t,i})\le2Z_t$, and the last inequality uses $Z_t\le1$ and $q_{t,i}/(q_{t,i}+\rho_{t,i})\le1$.

For the best arm, if $i^\star\neq \argmax_{i\in[K]}x_{t,i}$, the same calculation gives
\begin{align*}
 \eta_tx_{t,i^\star}^{3/2}\E[\ellhat_{t,i^\star}^2\mid\mathcal F_{t-1}]
 &\le\mu_{i^\star}\rho_{t,i^\star}
 =2\eta_t\mu_{i^\star}\sqrt{x_{t,i^\star}}
 \le2\eta_t\mu_{i^\star}\sum_{\substack{i=1\\i\ne i^\star}}^K\sqrt{x_{t,i}},
\end{align*}
where the last inequality follows because another coordinate is at least $x_{t,i^\star}$. Otherwise, we have
\begin{align*}
 \frac{q_{t,i}}{x_{t,i}}
 &=\left[\frac1{Z_t}-\frac{4\eta_t}{\sqrt{x_{t,i}}}\right]_+
 \le\left[\frac1{Z_t}-\frac{4\eta_t}{\sqrt{x_{t,i^\star}}}\right]_+
 =\frac{q_{t,i^\star}}{x_{t,i^\star}},
\end{align*}
for every $i\in[K]$, where the equalities follow from \pref{line:lf-sampling} and the inequality uses $x_{t,i}\le x_{t,i^\star}$. Multiplying by $x_{t,i}$ and summing gives
\begin{align*}
 1=\sum_{i=1}^Kq_{t,i}
 =\sum_{i=1}^Kx_{t,i}\frac{q_{t,i}}{x_{t,i}}
 \le\sum_{i=1}^Kx_{t,i}\frac{q_{t,i^\star}}{x_{t,i^\star}}
 =\frac{q_{t,i^\star}}{x_{t,i^\star}}.
\end{align*}
Thus $q_{t,i^\star}\ge x_{t,i^\star}$, and we have
\begin{align*}
 8\eta_t\sum_{\substack{i=1\\i\ne i^\star}}^Kx_{t,i}^{3/2}\E[\ellhat_{t,i^\star}^2\mid\mathcal F_{t-1}]
 &\le\frac{8\eta_tq_{t,i^\star}\mu_{i^\star}}{(q_{t,i^\star}+\rho_{t,i^\star})^2}\sum_{\substack{i=1\\i\ne i^\star}}^Kx_{t,i}^{3/2}\\
 &\le8\eta_t\mu_{i^\star}\sum_{\substack{i=1\\i\ne i^\star}}^K\frac{x_{t,i}}{q_{t,i^\star}}\sqrt{x_{t,i}}
 \le8\eta_t\mu_{i^\star}\sum_{\substack{i=1\\i\ne i^\star}}^K\sqrt{x_{t,i}},
\end{align*}
where the last inequality uses $x_{t,i}\le x_{t,i^\star}\le q_{t,i^\star}$. Combining the suboptimal-arm bounds with the two cases for arm $i^\star$ yields
\begin{align}
 &\E\left[\eta_t\min\left\{x_{t,I_t}^{3/2},8\sum_{i=1,i\ne I_t}^Kx_{t,i}^{3/2}\right\}\ellhat_{t,I_t}^2\mid\mathcal F_{t-1}\right]\nonumber\\
 &\le\sum_{\substack{i=1\\i\ne i^\star}}^K\eta_tx_{t,i}^{3/2}\E[\ellhat_{t,i}^2\mid\mathcal F_{t-1}]
 +\eta_t\min\left\{x_{t,i^\star}^{3/2},8\sum_{\substack{i=1\\i\ne i^\star}}^Kx_{t,i}^{3/2}\right\}\E[\ellhat_{t,i^\star}^2\mid\mathcal F_{t-1}]\nonumber\\
 &\le Z_t\sum_{\substack{i=1\\i\ne i^\star}}^K\frac{q_{t,i}\rho_{t,i}\Delta_i}{q_{t,i}+\rho_{t,i}}+10\eta_t\mu_{i^\star}\sum_{\substack{i=1\\i\ne i^\star}}^K\sqrt{x_{t,i}},
 \label{eq:bobw-expected-cost}
\end{align}
where the first inequality uses $\ellhat_{t,i}=0$ for $i\ne I_t$ and the $\mathcal F_{t-1}$-measurability of the coefficients, and the second uses the preceding coordinate bounds. This proves \pref{eq:bobw-inter-2}.

We next compute the difference between true and estimated regret. Expanding the conditional means gives
\begin{align}
 &\E\left[\ip{q_t-e_{i^\star}}{\ell_t}-\ip{x_t-e_{i^\star}}{\ellhat_t}\mid\mathcal F_{t-1}\right]\nonumber\\
 &=\ip{q_t}{\mu}-\mu_{i^\star}-\sum_{i=1}^K\frac{x_{t,i}q_{t,i}\mu_i}{q_{t,i}+\rho_{t,i}}+\frac{q_{t,i^\star}\mu_{i^\star}}{q_{t,i^\star}+\rho_{t,i^\star}}\nonumber\\
 &=\sum_{\substack{i=1\\i\ne i^\star}}^K\Delta_i\left(q_{t,i}-\frac{x_{t,i}q_{t,i}}{q_{t,i}+\rho_{t,i}}\right)
 +\mu_{i^\star}\left(\frac{q_{t,i^\star}}{q_{t,i^\star}+\rho_{t,i^\star}}-\sum_{i=1}^K\frac{x_{t,i}q_{t,i}}{q_{t,i}+\rho_{t,i}}\right),\label{eq:bias-bobw}
\end{align}
where the last equality substitutes $\mu_i=\mu_{i^\star}+\Delta_i$ for $i\ne i^\star$ and uses $\sum_{i=1}^Kq_{t,i}=1$. For the first summation, using \pref{eq:lf-matching}, we can obtain that
\begin{align*}
 \frac{x_{t,i}q_{t,i}}{q_{t,i}+\rho_{t,i}}
 &=\frac{x_{t,i}q_{t,i}}{q_{t,i}+2\rho_{t,i}}\frac{q_{t,i}+2\rho_{t,i}}{q_{t,i}+\rho_{t,i}}
 =Z_tq_{t,i}+Z_t\frac{q_{t,i}\rho_{t,i}}{q_{t,i}+\rho_{t,i}}.
\end{align*}
Substituting this identity into the first term in \pref{eq:bias-bobw} yields
\begin{align*}
 \sum_{\substack{i=1\\i\ne i^\star}}^K\Delta_i\left(q_{t,i}-\frac{x_{t,i}q_{t,i}}{q_{t,i}+\rho_{t,i}}\right)
 &=(1-Z_t)\sum_{\substack{i=1\\i\ne i^\star}}^K\Delta_iq_{t,i}
 -Z_t\sum_{\substack{i=1\\i\ne i^\star}}^K\frac{q_{t,i}\rho_{t,i}\Delta_i}{q_{t,i}+\rho_{t,i}}.
\end{align*}
For the second term in \pref{eq:bias-bobw}, direct calculation shows that
\begin{align*}
 \frac{q_{t,i^\star}}{q_{t,i^\star}+\rho_{t,i^\star}}-\sum_{i=1}^K\frac{x_{t,i}q_{t,i}}{q_{t,i}+\rho_{t,i}}
 &=1-\frac{\rho_{t,i^\star}}{q_{t,i^\star}+\rho_{t,i^\star}}-\sum_{i=1}^Kx_{t,i}+\sum_{i=1}^K\frac{x_{t,i}\rho_{t,i}}{q_{t,i}+\rho_{t,i}}\\
 &=\sum_{\substack{i=1\\i\ne i^\star}}^K\frac{x_{t,i}\rho_{t,i}}{q_{t,i}+\rho_{t,i}}-\frac{(1-x_{t,i^\star})\rho_{t,i^\star}}{q_{t,i^\star}+\rho_{t,i^\star}},
\end{align*}
where the first equality uses $q_{t,i}/(q_{t,i}+\rho_{t,i})=1-\rho_{t,i}/(q_{t,i}+\rho_{t,i})$ and the second uses $\sum_{i=1}^Kx_{t,i}=1$ and separates the term indexed by $i^\star$. Combining the two calculations gives
\begin{align}
 &\E\left[\ip{q_t-e_{i^\star}}{\ell_t}-\ip{x_t-e_{i^\star}}{\ellhat_t}\mid\mathcal F_{t-1}\right]\nonumber\\
 &=(1-Z_t)\sum_{\substack{i=1\\i\ne i^\star}}^K\Delta_iq_{t,i}
 -Z_t\sum_{\substack{i=1\\i\ne i^\star}}^K\frac{q_{t,i}\rho_{t,i}\Delta_i}{q_{t,i}+\rho_{t,i}}+\mu_{i^\star}\left(\sum_{\substack{i=1\\i\ne i^\star}}^K\frac{x_{t,i}\rho_{t,i}}{q_{t,i}+\rho_{t,i}}-\frac{(1-x_{t,i^\star})\rho_{t,i^\star}}{q_{t,i^\star}+\rho_{t,i^\star}}\right),
 \label{eq:bobw-signed-bias}
\end{align}
which proves \pref{eq:bobw-inter-3}.

We now substitute these bounds into the regret decomposition. Applying \pref{eq:bobw-inter}, \pref{eq:bobw-expected-cost}, and \pref{eq:bobw-signed-bias} gives
\begin{align*}
 \Reg_T
 &=\sum_{t=1}^T\E\left[\ip{x_t-e_{i^\star}}{\ellhat_t}\right]
 +\sum_{t=1}^T\E\left[\ip{q_t-e_{i^\star}}{\ell_t}-\ip{x_t-e_{i^\star}}{\ellhat_t}\right]\\
 &\le\frac{\Psi(e_{i^\star})-\Psi(x_1)}{\eta_1}
 +\E\left[\sum_{t=1}^T\left(\frac{1}{\eta_{t+1}}-\frac{1}{\eta_t}\right)(\Psi(e_{i^\star})-\Psi(x_{t+1}))\right]\\
 &\quad+\sum_{t=1}^T\E\left[Z_t\sum_{\substack{i=1\\i\ne i^\star}}^K\frac{q_{t,i}\rho_{t,i}\Delta_i}{q_{t,i}+\rho_{t,i}}
 +10\eta_t\mu_{i^\star}\sum_{\substack{i=1\\i\ne i^\star}}^K\sqrt{x_{t,i}}\right]\\
 &\quad+\sum_{t=1}^T\E\left[(1-Z_t)\sum_{\substack{i=1\\i\ne i^\star}}^K\Delta_iq_{t,i}
 -Z_t\sum_{\substack{i=1\\i\ne i^\star}}^K\frac{q_{t,i}\rho_{t,i}\Delta_i}{q_{t,i}+\rho_{t,i}}\right]\\
 &\quad+\sum_{t=1}^T\E\left[\mu_{i^\star}\left(\sum_{\substack{i=1\\i\ne i^\star}}^K\frac{x_{t,i}\rho_{t,i}}{q_{t,i}+\rho_{t,i}}
 -\frac{(1-x_{t,i^\star})\rho_{t,i^\star}}{q_{t,i^\star}+\rho_{t,i^\star}}\right)\right],
\end{align*}
The two sums involving $Z_tq_{t,i}\rho_{t,i}\Delta_i/(q_{t,i}+\rho_{t,i})$ cancel. For the last expectation, we have
\begin{align*}
 \mu_{i^\star}\left(\sum_{\substack{i=1\\i\ne i^\star}}^K\frac{x_{t,i}\rho_{t,i}}{q_{t,i}+\rho_{t,i}}-\frac{(1-x_{t,i^\star})\rho_{t,i^\star}}{q_{t,i^\star}+\rho_{t,i^\star}}\right)&\le\mu_{i^\star}\sum_{\substack{i=1\\i\ne i^\star}}^K\frac{x_{t,i}\rho_{t,i}}{q_{t,i}+\rho_{t,i}}\\
 &\le2Z_t\mu_{i^\star}\sum_{\substack{i=1\\i\ne i^\star}}^K\rho_{t,i}
 \le4\eta_t\mu_{i^\star}\sum_{\substack{i=1\\i\ne i^\star}}^K\sqrt{x_{t,i}},
\end{align*}
where the first inequality drops the nonpositive term, the second uses $x_{t,i}/(q_{t,i}+\rho_{t,i})\le2Z_t$, and the last uses $Z_t\le1$ and $\rho_{t,i}=2\eta_t\sqrt{x_{t,i}}$. Substituting this bound into the preceding regret inequality gives
\begin{align*}
 \Reg_T
 &\le\frac{\Psi(e_{i^\star})-\Psi(x_1)}{\eta_1}
 +\E\left[\sum_{t=1}^T\left(\frac{1}{\eta_{t+1}}-\frac{1}{\eta_t}\right)(\Psi(e_{i^\star})-\Psi(x_{t+1}))\right]\\
 &\quad+\sum_{t=1}^T\E\left[(1-Z_t)\sum_{\substack{i=1\\i\ne i^\star}}^K\Delta_iq_{t,i}\right]
 +14\sum_{t=1}^T\E\left[\eta_t\mu_{i^\star}\sum_{\substack{i=1\\i\ne i^\star}}^K\sqrt{x_{t,i}}\right],
\end{align*}
which finishes the proof.
\end{proof}

\end{document}